\documentclass{article} 
\usepackage[dvipsnames, svgnames, x11names]{xcolor}
\usepackage[utf8]{inputenc} 
\usepackage[T1]{fontenc}    
\usepackage{hyperref}       
\usepackage{url}            
\usepackage{booktabs}       
\usepackage{amsfonts}       
\usepackage{nicefrac}       
\usepackage{microtype}      
\usepackage{amsmath}
\usepackage{amssymb}
\usepackage{stmaryrd}
\usepackage{amsthm}
\usepackage{cleveref}
\usepackage{graphicx} 
\usepackage{caption}
\usepackage{mwe}
\usepackage{graphics,graphicx,epsfig,psfrag,subfigure,threeparttable,url}
\usepackage{float}
\usepackage{algorithm}
\usepackage{algorithmic}
\newtheorem{theorem}{Theorem}
\newtheorem{definition}{Definition}
\newtheorem{lemma}{Lemma}
\usepackage{color}
\usepackage{authblk}

\title{Applying Differential Privacy to Tensor Completion}
\author[1]{Zheng Wei}
\author[1]{Zhengpin Li}
\author[2]{Xiaojun Mao}
\author[1]{Jian Wang \thanks{Zheng Wei and Zhengpin Li contributed equally to this work, and Xiaojun Mao and Jian Wang are the co-corresponding authors. (Emails: zwei19@fudan.edu.cn, lizp21@m.fudan. edu.cn, maoxj@sjtu.edu.cn, jian\_wang@fudan.edu.cn)}}
\affil[1]{School of Data Science, Fudan University, China }
\affil[2]{School of Mathematical Sciences, Shanghai Jiao Tong University, China}
\date{}
\begin{document}
\maketitle

\begin{abstract}

Tensor completion aims at filling the missing or unobserved entries based on partially observed tensors. However, utilization of the observed tensors often raises serious privacy concerns in many practical scenarios. To address this issue, we propose a solid and unified framework that contains several approaches for applying differential privacy to the two most widely used tensor decomposition methods: i) CANDECOMP/PARAFAC~(CP) and ii) Tucker decompositions. For each approach, we establish a rigorous privacy guarantee and meanwhile evaluate the privacy-accuracy trade-off.  Experiments on synthetic and real-world datasets demonstrate that our proposal achieves high accuracy for tensor completion while ensuring strong privacy protections.
\end{abstract}

\section{Introduction}\label{sec:intro}

In machine learning knowledge, missing data is a prevalent issue, which can be caused by data collection, data corrosion, or other artificial reasons. As one of the most popular completion methods, low-rank matrix completion has received much attention in a wide range of applications, such as collaborative filtering~\cite{goldberg1992using}, computer vision~\cite{tomasi1992shape}, and multi-class learning~\cite{amit2007uncovering, evgeniou2007multi}. However, there are many genuine cases where data has more than two dimensions and are best represented as multi-way arrays, such as tensor. For instance, electronic health records~(EHRs)~\cite{johnson2016mimic}, which reserve patients' clinical histories, consist of three parts: patients, diagnosis, and procedure. A more common scenario is that data contains the time dimension, such as traffic data of network~\cite{vardi1996network}, which can be viewed as a series traffic matrix presenting the volumes of traffic between original and destination pairs by unfolding as time intervals. Therefore, as a natural high-order extension of low-rank matrix completion, low-rank tensor completion is gaining more and more interest. 

For completion methods, privacy-preserving is a significant issue that cannot be ignored. This concept is firstly proposed in~\cite{agrawal2000privacy} and considered as a vital goal for mining the value of data while protecting its privacy. In recent years, this issue has attracted increasing attention in matrix and tensor completions as well as their applications. For example, users are required to offer their ratings to recommender service in recommendation scenarios, which often raises serious privacy concerns because of insidious attacks and unexpected inference on users' ratings or profiles~\cite{narayanan2008robust, aimeur2008lambic, mcsherry2009differentially}. The purpose of privacy-preserving tensor completion is to ensure high-level privacy of the observed data, while keeping completion performance as high as possible. 

To the best of our knowledge, few studies systematically studied privacy-preserving tensor completion. In this work, we propose a solid and unified framework for two most widely used tensor decomposition methods: CANDECOMP/PARAFAC~(CP) decomposition~\cite{hitchcock1927expression, carroll1970analysis, harshman1970foundations} and Tucker decomposition~\cite{tucker1966some, kroonenberg1980principal, de2000multilinear} to maintain privacy guarantees by utilizing differential privacy~\cite{dwork2006calibrating, dwork2014algorithmic}, the dominant standard for privacy protection. The framework contains several privacy-preserving computation ways: input perturbation, gradient perturbation, and output perturbation. They all result in the trade-off between the accuracy and privacy-preserving. 

The contributions of our work are summarized as follows. 

\begin{itemize}
    \item 
We are the first to propose a solid and unified framework for applying differential privacy to tensor completion.  

\item We provide complete algorithm procedures and theoretical analysis for each privacy-preserving approach in our framework. 

\item Experimental results on synthetic and real-world datasets demonstrate that the proposed approaches can yield high accuracy, while ensuring strong privacy protections. 
\end{itemize}

\section{Related Work}\label{sec:related}

Differential privacy has drawn much attention for privacy-preserving data analysis because of its provable privacy guarantees and few computation~\cite{dwork2006calibrating}. Various work took this definition as the standard in data mining~\cite{friedman2010data}, recommendation systems~\cite{shen2014privacy, friedman2016differential, 2015Differentially, liu2015fast}, and deep learning~\cite{abadi2016deep}. Under the constraint of differential privacy, there are numerous algorithms for achieving the trade-off between the level of privacy-preserving and accuracy. Stimulated by~\cite{dwork2014algorithmic, friedman2016differential}, we come up with a complete framework for tensor completion with differential privacy. In~\cite{friedman2016differential}, input perturbation was proposed to utilize the noise from the Laplacian mechanism to interference input data. Among the analysis for gradient perturbation approach, Williams {\it et al.}~\cite{williams2010probabilistic} firstly conducted investigation on gradient descent with noisy updates. Afterwards, ~\cite{bassily2014private, wang2018differentially, jayaraman2018distributed} developed a set of gradient perturbation algorithms and established privacy guarantees for them. Objective perturbation was firstly proposed in~\cite{chaudhuri2008privacy} and extended in~\cite{chaudhuri2011differentially, 2015Differentially}. This method aims at perturbing the objective function before training. Last but not least, output perturbation~\cite{dwork2006calibrating, chaudhuri2011differentially, jain2012differentially} works by adding noise to the solved optimal values. ~\cite{wu2017bolt, zhang2017efficient} provided novel algorithms and convergence analysis from the viewpoint of optimization.

Our work focuses on the combination of differential privacy and tensor completion. As a natural extension of matrix completion, tensor completion has also been used in many applications, such as data mining~\cite{morup2011applications}. CP decomposition, as a classical and notable algorithm, was first proposed by Hitchcock~\cite{hitchcock1927expression} and further discussed in~\cite{carroll1970analysis, harshman1970foundations}. Another representational decomposition algorithm, Tucker decomposition, was firstly presented by Tucker~\cite{tucker1966some} and further developed in~\cite{kroonenberg1980principal, de2000multilinear}. There are also several other tensor decomposition methods related to CP and Tucker listed in~\cite{kolda2009tensor}. In~\cite{liu2012tensor}, Liu {\it et al.} firstly built the theoretical foundation of low-rank tensor completion and propose three approaches based on the novel definition of trace norm for tensors: SiLRTC, FaLRTC, and HaLRTC. There have also been many follow-up progress of~\cite{liu2012tensor}, see, e.g.,~\cite{xu2013parallel, wright2009robust, tan2014tensor, zhou2017tensor, song2019tensor}. 

Some previous studies resolved privacy-preserving tensor completion problems under specific circumstances. Wang and Anandkumar~\cite{wang2016online} proposed an algorithm for differentially private tensor decomposition using a noise calibrated tensor power method. Imtiaz and Sarwate~\cite{imtiaz2018distributed} designed two algorithms for distributed differentially private principal component analysis and orthogonal tensor decomposition. Ma {\it et al.}~\cite{ma2019privacy} developed a novel collaborative tensor completion method that can preserve privacy on EHRs. Yang {\it et al.}~\cite{yang2021optimized} presented a privacy-preserving tensor completion method that uses the optimized federated soft-impute algorithm, which can provide privacy guarantees on cyber-physical-social systems. In this paper, we take CP decomposition and Tucker decomposition as the backbone completion algorithms to combine with differential privacy.
\section{Preliminaries and Notations}\label{sec:pre}

In this section, we introduce the notations and preliminaries about tensor completion and differential privacy throughout the paper, and state some known lemmas that will be utilized later.
 
We describe tensor and its operations mainly based on notations in~\cite{kolda2009tensor}. Tensors are denoted by Euler script letters ($\mathcal{X, Y, Z}$), matrices by boldface capital letters ($\mathbf{A, B, C}$), vectors by boldface lowercase letters ($\mathbf{a, b, c}$) and scalars by vanilla lowercase letters (a, b, c). Considering entry representation, we use $x_{ijk}$ to represent the value seated in $(i, j, k)$ of $\mathcal{X}$. We also use subscript ``$:$'' to indicate all values of a dimension. For instance, $\mathbf{a}_{m:}$ and $\mathbf{a}_{:r}$ are the $m$th-row and $r$th-column of matrix $\mathbf{A}$, respectively. 

\paragraph{Hadamard Product:} The Hadamard product is the elementwise product for two $n$th-order tensors with the same size. For example, the Hadamard product for two tensors $\mathcal{X}\in \mathbb{R}^{i_1\times\cdots\times i_n}$ and $\mathcal{Y}\in \mathbb{R}^{i_1\times\cdots\times i_n}$ is denoted by $\mathcal{X} * \mathcal{Y}$, which is defined by $(\mathcal{X} * \mathcal{Y})_{i_1\cdots i_n}=x_{i_1\cdots i_n} y_{i_1\cdots i_n}$.

\paragraph{Khatri-Rao Product:} The Khatri-Rao product of matrices $\mathbf{A} \in \mathbb{R}^{I\times L}$ and $\mathbf{B} \in \mathbb{R}^{K\times L}$ is denoted by $\mathbf{A} \odot \mathbf{B}$, which is defined by $\mathbf{A} \odot \mathbf{B}=\left[\mathbf{a}_{:1} \otimes \mathbf{b}_{:1} \cdots \mathbf{a}_{:L} \otimes \mathbf{b}_{:L}\right]_{I J \times L}$ where $\otimes$ denotes Kronecker product. The Kronecker product of two vectors $\mathbf{a} \in \mathbb{R}^{I}$ and $\mathbf{b} \in \mathbb{R}^{J}$ is obtained by $\mathbf{a}  \otimes \mathbf{b}  = [a_{1} \mathbf{b} \ a_{2} \mathbf{b}  \cdots a_{I} \mathbf{b} ]^T$.

\paragraph{Mode-$n$ Product:} The mode-$n$ product of a tensor $\mathcal{X} \in \mathbb{R}^{I_{1} \times I_{2} \times \cdots \times I_{N}}$ and a matrix $\mathbf{U} \in \mathbb{R}^{J \times I_{n}}$ is denoted by $\mathcal{X} \times_{n} \mathrm{U}$, which is of size $I_{1} \times \cdots \times I_{n-1} \times J \times I_{n+1} \times \cdots \times I_{N}$. Elementwise, we have $\left(\mathcal{X} \times{ }_{n} \mathbf{U}\right)_{i_{1} \cdots i_{n-1} j i_{n+1} \cdots i_{N}}=\sum_{i_{n}=1}^{I_{n}} x_{i_{1} i_{2} \cdots i_{N}} u_{j i_{n}}$

\paragraph{CP Decomposition:} The standard CP decomposition factorizes a tensor into a sum of component rank-one tensors. Given a tensor $\mathcal{X} \in \mathbb{R}^{I\times J \times K}$, we have
\[\mathcal{X} \approx \sum_{r=1}^{R} \mathbf{a}_{: r}^{(1)} \circ \cdots \circ \mathbf{a}_{: r}^{(n)}=\llbracket \mathbf{A}^{(1)}, \ldots, \mathbf{A}^{(n)} \rrbracket,\]
where $R$ denotes the rank of tensor and $\mathbf{A}^{(n)}$ is the $n$-mode factor matrix consisting of $R$ columns representing $R$ latent components which can be represented as $\mathbf{A}^{(n)}=[\mathbf{a}_{: 1}^{(n)} \cdots \mathbf{a}_{: R}^{(n)}].$
  
\paragraph{Tucker Decomposition:} The standard Tucker decomposition factorizes a tensor into a core tensor multiplied by a matrix along each mode. For a tensor $\mathcal{X} \in \mathbb{R}^{I\times J \times K}$, we can express it by
\[\mathcal{X} \approx \mathcal{G} \times_{1} \mathbf{~A} \times{ }_{2} \mathbf{~B} \times{ }_{3} \mathbf{C}=\sum_{p=1}^{P} \sum_{q=1}^{Q} \sum_{t=1}^{T} g_{p q t} \mathbf{a}_{:p} \circ \mathbf{b}_{:q} \circ \mathbf{c}_{:t}=\llbracket \mathcal{G} ; \mathbf{A}, \mathbf{B}, \mathbf{C} \rrbracket\]
where $\mathcal{G}\in \mathbb{R}^{P \times Q \times T}$ and $g_{p q t}$ indicate the core tensor and the element of $\mathcal{G}$ on coordinate $(p, q, t)$ respectively, and $\mathbf{A} \in \mathbb{R}^{I \times P}$, $\mathbf{B} \in \mathbb{R}^{J \times Q}$ and $\mathbf{C} \in \mathbb{R}^{K \times T}$ represent the factor matrices.

\subsection{Differential Privacy}
	
\begin{definition}
Let $f:\mathbb{R}^d \mapsto \mathbb{R}$ be a function:
\begin{itemize}

\item $f$ is $L$-Lipschitz if for any $u, v \in \mathbb{R}^d$, $\|f(u)-f(v)\| \leq L\|u-v\|$;
\item $f$ is $\beta$-smooth if $\|\nabla f(u)-\nabla f(v)\| \leq \beta\|u-v\|$, where $\nabla$ denotes the first order derivative.
\end{itemize}

\end{definition}

\begin{definition}
	A (randomized) algorithm $\mathcal{A}$ whose outputs lie in a domain $\mathcal{S}$ is said to be $\epsilon$-differentially private if for all subsets $S \subseteq \mathcal{S}$, for all datasets $\mathcal{D}$ and $\mathcal{D}^{\prime}$ that differ in at most one entry, it holds that: 
	\begin{equation}
	\label{eq: dp}
		\operatorname{Pr}(\mathcal{A}(\mathcal{D}) \in S) \leq e^{\epsilon} \operatorname{Pr}\left(\mathcal{A}\left(\mathcal{D}^{\prime}\right) \in S\right).
	\end{equation}
\end{definition}

\begin{definition}
	The $L_{p}$-sensitivity of a function $f: \mathcal{D}^{n} \rightarrow \mathbb{R}^{d}$ is the smallest number $\Delta_p(f)$ such that for all $\mathbf{x}, \mathbf{x}^{\prime} \in \mathcal{D}^{n}$ which differ in a single entry,
	\begin{equation}
	\label{eq: lp-sensitivity}
		\left\|f(\mathbf{x})-f\left(\mathbf{x}^{\prime}\right)\right\|_{p} \leq \Delta_p(f),
	\end{equation}
\end{definition}
where $\Delta_p(f)$ captures the magnitude by which a single individual’s data can change the function $f$ in the worst case, which provides an upper bound on how much we must perturb the input to preserve privacy.

\begin{lemma}[Laplace Mechanism~\cite{dwork2006calibrating}]
\label{le1}
	For all $f: \mathcal{D}^{n} \rightarrow \mathbb{R}^{d}$, the Laplace mechanism is defined 
\begin{equation}
\label{eq: lapalce-mechanism}
		\operatorname{San}_{f}(\mathbf{x})=f(\mathbf{x})+\left(Y_{1}, \ldots, Y_{d}\right),
	\end{equation}
	which ensures $\epsilon$-differential privacy, where the $Y_{i}$ are i.i.d. drawn from $\operatorname{Lap}(\Delta_1(f) / \epsilon)$.
\end{lemma}

\begin{lemma}[Exponential Mechanism~\cite{dwork2006calibrating}]
\label{le2}
	Let $f$ be a deterministic query that maps a database to a vector in $\mathbb{R}^{d}$. Then publishing $f(\mathcal{D})+\boldsymbol{\kappa}$ where $\boldsymbol{\kappa}$
	is sampled from the distribution with density
	\begin{equation}
    \label{eq: exponential-mechanism}
		p(\boldsymbol{\kappa}) \propto \exp \left(-\frac{\varepsilon\|\boldsymbol{\kappa}\|}{\Delta_2(f)}\right),
	\end{equation}
	preserves $\epsilon$-differential privacy.
\end{lemma}

\begin{lemma}[Private Convex Permutation-based SGD~\cite{wu2017bolt}] 
\label{le5}
Consider $\tau$-passes private convex SGD~(i.e. PSGD) for $L$-Lipschitz, convex and $\beta$-smooth optimization. Suppose further that we have constant learning rate $\eta_1 = \eta_2 = \cdots \eta_{\tau} = \eta \leq \frac{2}{\beta}$. Denote $S$ and $S^\prime$ as two datasets differing on one single entry, $\tau$ as number of iteration, we have $\sup_{S \sim S^{\prime}} \sup _{r} \Delta_{T} \leq 2 \tau L \eta$, where $r$ indicates a random permutation for datasets and $T$ represents the number of iterations.
\end{lemma}

\section{Differential Privacy Tensor Completion}\label{sec:method}

In this section, we introduce the proposed framework for privacy-preserving tensor completion. We focus on the CP and Tucker decompositions with several privacy-preserving approaches via stochastic gradient descent~(SGD) under the constraints of differential privacy. Considering the stages of tensor completion, we design input, gradient, and output perturbation approaches to maintain privacy, respectively. The overall framework is shown in Figure~\ref{fig0}.

\begin{figure}[htbp]
	\centering
	 \includegraphics[width=.9\linewidth]{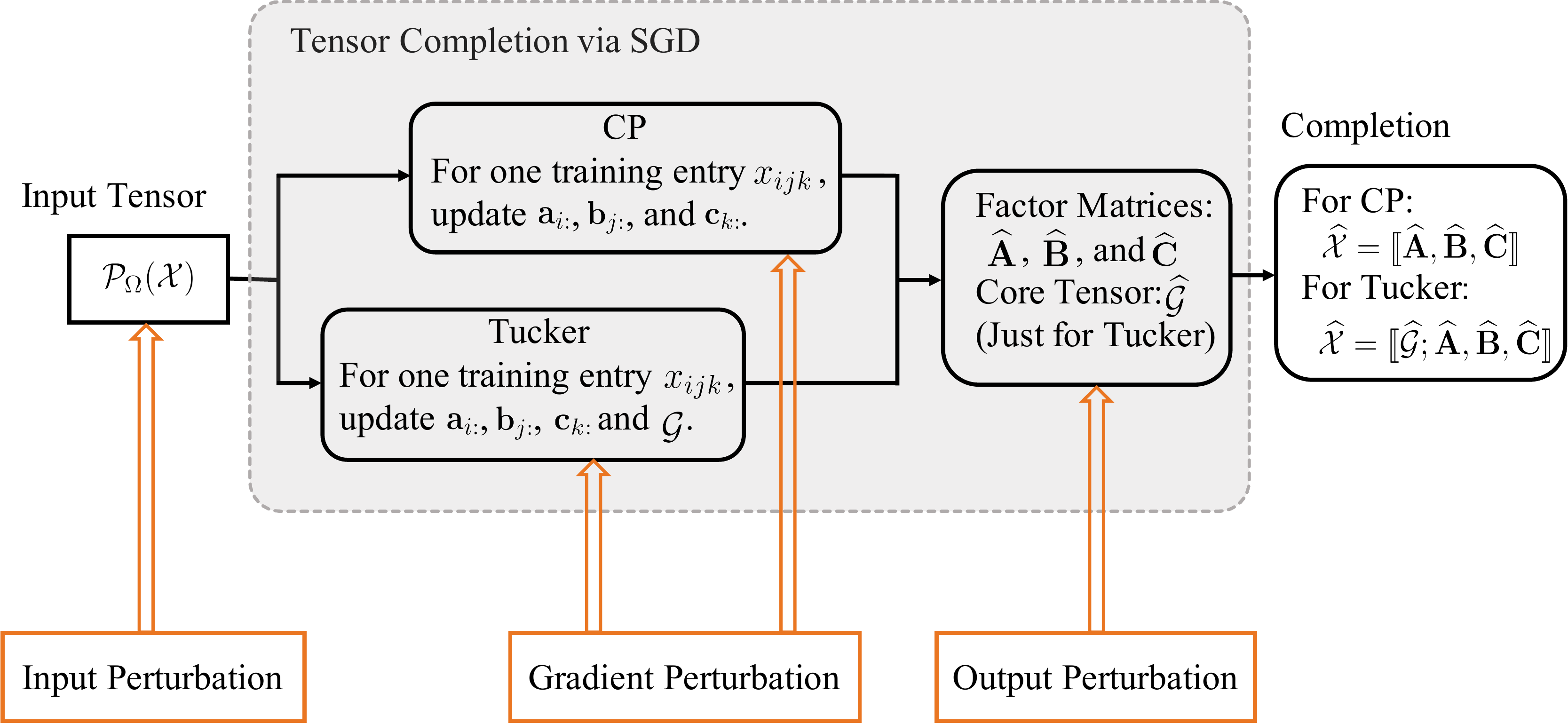}
    \caption{Various perturbation approaches within tensor completion framework.}
	\label{fig0}
\end{figure}

\subsection{Problem Formulation}
From now on, $\mathcal{X}\in \mathbb{R}^{n_{1} \times n_{2} \times n_{3}}$, which is generated by true tensor $\widetilde{\mathcal{X}}$ with unknown noise, represents the noisy incomplete tensor used to obtain estimated factor matrices and core tensor. We denote observation set by $\Omega$ which contains the indexes of available entries, and $x_{ijk}$ is observed if and only if $(i, j, k) \in \Omega$. For convenience, we introduce the sampling operator $\mathcal{P}_{\Omega}: \mathbb{R}^{n_{1} \times n_{2} \times n_{3}} \rightarrow \mathbb{R}^{n_{1} \times n_{2} \times n_{3}}$:
\begin{equation}
\label{eq: proj-operator}
\left[\mathcal{P}_{\Omega}(\mathcal{X})\right]_{i j k}=\left\{\begin{array}{ll}x_{i j k}, & (i, j, k) \in \Omega \\ 0, & \text { otherwise.}\end{array}\right.
\end{equation}
Denote three latent matrices derived from factorization by $\mathbf{A} \in \mathbb{R}^{n_1 \times d} $, $\mathbf{B} \in \mathbb{R}^{n_2 \times d}$, and $\mathbf{C} \in \mathbb{R}^{n_3 \times d}$ where $d$ indicates the rank of $\widetilde{\mathcal{X}}$, and the CP based completion problem can be formulated as:
\begin{equation}
\label{eq: cp-optimization}
  \underset{\mathbf{A}, \mathbf{B}, \mathbf{C}}{\operatorname{\min}} \quad f(\mathbf{A}, \mathbf{B}, \mathbf{C}) = \|\mathcal{P}_{\Omega} (\mathcal{X}-\llbracket \mathbf{A}, \mathbf{B}, \mathbf{C} \rrbracket)\|_F^{2} + \lambda(\|\mathbf{A}\|^2_F + \|\mathbf{B}\|^2_F + \|\mathbf{C}\|^2_F),
\end{equation}
where $\lambda$ acts as a regularization parameter to control a tunable tradeoff between fitting errors and encouraging low-rank tensor. In terms of Tucker decomposition, we denote the core tensor by $\mathcal{G}$, and set the size of $\mathcal{G}$ to $d \times d \times d$ for simplicity. In a similar regularization manner for factor matrices, we impose a $F$-norm penalty to restrict the complexity of the core tensor. Thereby, we can reformulate the problem~\eqref{eq: cp-optimization} as:
\begin{equation}
\label{eq: tucker-optimization} 
    \underset{\mathbf{A}, \mathbf{B}, \mathbf{C}, \mathcal{G}}{\operatorname{\min}}~f(\mathbf{A}, \mathbf{B}, \mathbf{C}, \mathcal{G}) = \|\mathcal{P}_{\Omega} (\mathcal{X}-\llbracket \mathcal{G} ; \mathbf{A}, \mathbf{B}, \mathbf{C} \rrbracket)\|_F^{2} + \lambda_o(\|\mathbf{A}\|^2_F + \|\mathbf{B}\|^2_F + \|\mathbf{C}\|^2_F) + \lambda_g\|\mathcal{G}\|^2_F,  
\end{equation}
where $\lambda_o$ and $\lambda_g$ indicate regularization parameters for the factor matrices and the core tensor, respectively. The presence of the core tensor constitutes the main difference between these two decomposition methods. CP decomposition performs computationally more flexible in dealing with large-scale datasets, whereas Tucker decomposition is more general and effective because its core tensor can capture complex interactions among components that are not strictly trilinear~\cite{song2019tensor}. Consequently, we can consider CP decomposition as a special case of Tucker decomposition where the cardinalities of the dimensions of latent matrices are equal and the off-diagonal elements of the core tensor are zero~\cite{schein2016bayesian}. In the following parts, we provide theoretical analysis and algorithm procedures of  the perturbation mechanisms based on Tucker decomposition.

\subsection{Private Input Perturbation}
In this approach, each entry of input tensor $\mathcal{X}$ is considered independently of the rest, and is perturbed by noise, which is bounded by $L_1$-sensitivity of $\mathcal{X}$. Suppose the entries of $\mathcal{X}$ are in the range of $[\mathcal{X}_{\max}, \mathcal{X}_{\min}]$, the $L_1$-sensitivity of the tensor is $\Delta^{(I)}_\mathcal{X} = \mathcal{X}_{\max} -\mathcal{X}_{\min}$, and noises are sampled from $\operatorname{Lap}(\Delta^{(I)}_\mathcal{X}  / \epsilon)$. The outline of this process is shown in Algorithm~\ref{Al1}.

\begin{algorithm}
\caption{Private Input Perturbation}
\label{Al1}
	\begin{algorithmic}[1]
		\REQUIRE $\mathcal{X}$: noisy incomplete tensor, $\Omega$: indexes set of observations, $d$: rank of tensor, $\lambda_o$: regularization parameter for the factor matrices, $\lambda_g$: regularization parameter for the core tensor, $\epsilon$: privacy budget

		\STATE Generate each entry of noise tensor $\mathcal{N}$ by $\operatorname{Lap}(\Delta^{(I)}_\mathcal{X}  / \epsilon)$
		\STATE Let $\mathcal{X}^{\prime} = \{x_{ijk} + n_{ijk} | (i, j,k) \in \Omega\}$

		\STATE Use $\mathcal{X}^{\prime}$ as input to solve~\eqref{eq: tucker-optimization} via SGD and obtain estimated $\widehat{\mathbf{A}}, \widehat{\mathbf{B}}, \widehat{\mathbf{C}}$ and $\widehat{\mathcal{G}}$
		\ENSURE Estimated $\widehat{\mathbf{A}} \in \mathbb{R}^{n_1 \times d} $, $\widehat{\mathbf{B}} \in \mathbb{R}^{n_2 \times d}$, $\widehat{\mathbf{C}} \in \mathbb{R}^{n_3 \times d}$ and $\widehat{\mathcal{G}} \in \mathbb{R}^{d \times d \times d} $
	\end{algorithmic}
\end{algorithm}
\begin{theorem}
\label{the1}
	Algorithm~\ref{Al1} maintains $\epsilon$-differential privacy.
\end{theorem}
\begin{proof}[Proof of Theorem~\ref{the1}]
	The $L_1$-sensitivity of the input tensor is $\Delta^{(I)}_\mathcal{X}  = \mathcal{X}_{\max} -\mathcal{X}_{\min}$. According to Lemma~\ref{le1}, this algorithm maintains $\epsilon$-differential privacy.\end{proof}
Essentially, exerting private perturbation on $\mathcal{X}$ is equivalent to adding noise following a specific distribution to it. The magnitude of noises is determined by the $L_1$-sensitivity of $\mathcal{X}$ and privacy budget $\epsilon$. Optionally, to limit the influence of excessive noise, we can clamp perturbed $\mathcal{X}$ to a fixed range before training. Moreover, the input perturbation can protect the privacy concerning the existence of observations in scenarios where missing entries are assigned to zero by default. 

\subsection{Private Gradient Perturbation}

The gradient perturbation maintains privacy by introducing noise in the SGD step~\cite{bassily2014private}. In our gradient perturbation, we add noises to the computed gradients, and then utilize noisy gradients to update the corresponding rows of the factor matrices and the core tensor. For the sake of simplicity, we spend the all privacy budget on one single factor matrix. Here, we take $\mathbf{C}$ as an example. In each iteration, the gradient of $\mathbf{C}$ will be added by noise sampled from one exponential distribution. Before that, to be compatible with our theoretical assumption in Theorem~\ref{the2}, we clip the gradient $l_2$-norms of $\mathbf{C}$ to a constant $m$ using $\mathbf{v} \leftarrow \mathbf{v} / \max \left(1,\|\mathbf{v}\|_{2} / m\right)$~\cite{abadi2016deep, wang2019dp}. The global sensitivity here is denoted by $\Delta^{(G)}_{\mathcal{X}}$. Algorithm~\ref{Al2} summarizes this process. 
\begin{theorem}
\label{the2}
	Suppose that function $f$ with regard to $\mathbf{C}$ in~\eqref{eq: tucker-optimization} is $L$-Lipschitz, Algorithm~\ref{Al2} maintains $\epsilon$-differential privacy.
\end{theorem}
\begin{algorithm}
	\caption{Private Gradient Perturbation}
	\label{Al2}
	\begin{algorithmic}[1]
	\REQUIRE $\mathcal{X}$: noisy incomplete tensor, $\Omega$: indexes set of observations, $d$: rank of tensor, $\lambda_o$: regularization parameter for the factor matrices, $\lambda_g$: regularization parameter for the core tensor, $n$: number of iterations, $\epsilon$: privacy budget, $\eta$: learning rate, $m$: clipping constant

		\STATE Initialize random factor matrices $\mathbf{A}, \mathbf{B}, \mathbf{C}, \mathcal{G}$
		\FOR{$n$ iterations}
			\FOR{$x_{ijk} \in \mathcal{X}$}
				\STATE $\mathbf{a}_{i :} \leftarrow \mathbf{a}_{i :} - \eta \nabla_{\mathbf{a}_{i:}} f$
				\STATE $\mathbf{b}_{j:} \leftarrow \mathbf{b}_{j :} - \eta \nabla_{\mathbf{b}_{j:}} f$
                \STATE $\nabla_{\mathbf{c}_{k:}}f \leftarrow \nabla_{\mathbf{c}_{k:}}f / \max \left(1,\|\nabla_{\mathbf{c}_{k:}}f\|_{2} / m\right)$ 
                \STATE Sample noise vector $\mathbf{n}_{i:}$ satisfying $p(\mathbf{n}_{i:}) \propto e^{-\frac{\varepsilon\|\mathbf{n}_{i:}\|}{\Delta^{(G)}_{\mathcal{X}}}}$
                
				\STATE $\mathbf{c}_{k:} \leftarrow \mathbf{c}_{k:}  - \eta (\nabla_{\mathbf{c}_{k:}} f + \mathbf{n}_{i:})$
				\STATE $\mathcal{G} \leftarrow \mathcal{G}-\eta \nabla_{\mathcal{G}} f$
			\ENDFOR
		\ENDFOR
		\ENSURE Estimated $\widehat{\mathbf{A}} \in \mathbb{R}^{n_1 \times d} $, $\widehat{\mathbf{B}} \in \mathbb{R}^{n_2 \times d}$, $\widehat{\mathbf{C}} \in \mathbb{R}^{n_3 \times d}$ and $\widehat{\mathcal{G}} \in \mathbb{R}^{d \times d \times d} $
	\end{algorithmic}
\end{algorithm}
\begin{proof}
	Let $\mathcal{X}$ and $\mathcal{X}^{\prime}$ be two tensors differing at only element $x_{pqr}$ and $x_{pqr}^{\prime}$. Let $\mathbf{N} = \{n_{ij}\}$ and $\mathbf{N}^\prime = \{n^\prime_{ij}\}$ be the noise matrices when training with $\mathcal{X}$ and $\mathcal{X}^{\prime}$ respectively. According to the optimization formulation~\eqref{eq: tucker-optimization}, it is obviously differentiable anywhere, which ensures the unique mapping from input to output. Denote $\mathbf{C}^*$ as the derived factor matrix minimizes both the optimization problems, and we have $\forall k \in\{1,2, \cdots, n_3\}$, $\nabla_{\mathbf{c}_{k:}} f(\mathbf{c}_{k:}^*|\mathcal{X})=\nabla_{\mathbf{c}_{k:}} f(\mathbf{c}_{k:}^*|\mathcal{X}^{\prime})$. Thereby, given $x_{ijk}$ and $x_{ijk}^{\prime}$, we have:
\begin{equation}
\label{eq: the2-eq1}
\nabla_{\mathbf{c}_{k:}} f(\mathbf{c}_{k:}^*|\mathcal{X}) + \mathbf{n}_{k:}  
    =
\nabla_{\mathbf{c}_{k:}} f(\mathbf{c}_{k:}^*|\mathcal{X}^{\prime}) + \mathbf{n}^\prime_{k:}.
\end{equation}
Then we can derive that:
\begin{equation}
\label{eq: the2-eq2}
\begin{split}
  \mathbf{n}_{k:} - \mathbf{n}^\prime_{k:} 
  &= \nabla_{\mathbf{c}_{k:}} f(\mathbf{c}_{k:}^*|\mathcal{X}) - \nabla_{\mathbf{c}_{k:}} f(\mathbf{c}_{k:}^*|\mathcal{X}^{\prime}), \\
  \|\mathbf{n}_{k:} - \mathbf{n}^\prime_{k:}\|
  = \;&\|\nabla_{\mathbf{c}_{k:}} f(\mathbf{c}_{k:}^*|\mathcal{X}) - \nabla_{\mathbf{c}_{k:}} f(\mathbf{c}_{k:}^*|\mathcal{X}^{\prime})\| \leq 2L.
  \end{split}
\end{equation}
Denote $\Delta^{(G)}_\mathcal{X} = 2L$. For any pair of $x_{pqg}$ and ${x^{\prime}}_{pqg}$, we have:
\begin{equation}
\label{eq: the2-eq3}
\begin{split}
  \frac{\text{Pr}\left[\mathbf{C}= \mathbf{C}^* \mid \mathcal{X}\right]}{\text{Pr}\left[\mathbf{C}= \mathbf{C}^* \mid \mathcal{X}^{\prime}\right]} 
  & = \prod_{k=1}^{n_3} \frac{ p(\mathbf{n}_{k:})}{p(\mathbf{n}_{k:}^\prime)}  = \exp \left\{-\frac{\epsilon\left(\sum_{k=1}^{n_3} \|\mathbf{n}_{k:}\| -\sum_{k=1}^{n_3}\|\mathbf{n}_{k:}^\prime\|\right)} {\Delta^{(G)}_\mathcal{X}} \right\}\\
  & = \exp \left\{  - \frac{\epsilon( \|\mathbf{n}_{k:}\| -\|\mathbf{n}_{k:}^\prime\|)}{\Delta^{(G)}_\mathcal{X}}\right\}   \leq \exp \left\{  \frac{ \epsilon (\|\mathbf{n}_{k:} - \mathbf{n}_{k:}^\prime\|)}{\Delta^{(G)}_\mathcal{X}}\right\}  \\
& \leq \exp(\epsilon).
\end{split}
\end{equation}
Hence, the algorithm maintains $\epsilon$-differential privacy for the whole process\end{proof}
In contrast to the previous gradient perturbation approaches~\cite{bassily2014private}, we propose a novel proof to separate the privacy budget from the iteration number. In this way, we can avoid generating excessive noise under a too-small privacy budget in each iteration. Besides, we set clipping constant to bound the gradient $l_2$-norm to limit fluctuation of gradient and magnitude of noise, which makes the updating process more robust regarding privacy budget.

\subsection{Private Output Perturbation}

The output perturbation achieves privacy protections by adding noise to the final model~\cite{dwork2006calibrating}. We can divide the privacy budget among all outputs in our approach, including the factor matrices and the core tensor. For simplicity, we only consider adding noise to estimated $\widehat{\mathbf{C}}$. After the updating process of SGD, noise vectors sampled by one exponential mechanism will be added to each row of $\widehat{\mathbf{C}}$. Define $\Delta^{(O)}_{\mathcal{X}} = 2\tau L\eta$ where $\tau$, $L$, and $\eta$ indicate number of iterations, Lipschitz constant, and learning rate, respectively. The summary of this process is shown in Algorithm~\ref{Al3}.
\begin{algorithm}
	\caption{Private Output Perturbation}
	\label{Al3}
	\begin{algorithmic}[1]
		\REQUIRE $\mathcal{X}$: noisy incomplete tensor, $\Omega$: indexes set of observations, $d$: rank of tensor, $\epsilon$: privacy budget
		\STATE Solve~\eqref{eq: tucker-optimization} via SGD and obtain estimated $\widehat{\mathbf{A}}, \widehat{\mathbf{B}},  \widehat{\mathbf{C}}$ and $\widehat{\mathcal{G}}$
		\STATE Sample noise matrix $\mathbf{N}$, all rows of which are sampled from $\exp\left\{-\frac{\epsilon\|\mathbf{n}_{i:}\|}{\Delta^{(O)}_{\mathcal{X}}}\right\}$
		\STATE $\widehat{\mathbf{C}} \leftarrow \widehat{\mathbf{C}} + \mathbf{N}$
		     \ENSURE Estimated $\widehat{\mathbf{A}} \in \mathbb{R}^{n_1 \times d} $, $\widehat{\mathbf{B}} \in \mathbb{R}^{n_2 \times d}$, $\widehat{\mathbf{C}} \in \mathbb{R}^{n_3 \times d}$ and $\widehat{\mathcal{G}} \in \mathbb{R}^{d \times d \times d} $
	\end{algorithmic}
\end{algorithm}
\begin{theorem}
\label{the3}
	Algorithm~\ref{Al3} maintains $\epsilon$-differential privacy.
\end{theorem} 
\begin{proof}[Proof of Theorem~\ref{the3}]
According to Lemma~\ref{le5}, $L_2$-sensitivity is bounded by $2 \tau L\eta$ denoted as $\Delta^{(O)}_{\mathcal{X}}$. By adding noises from~\ref{le2}, it directly yields $\epsilon$-differential privacy for this algorithm.\end{proof}
The advantages of the output perturbation lie in its flexible allocation of privacy budget and ease of implementation. On the other hand, the introduced noise directly impacts completion results, which makes completion performance susceptible to the noise.

\section{Evaluation}
\label{sec:eva}

In this section, we evaluate our proposal on synthetic and real-world datasets. For each experiment scenario, we randomly split observations into $80\%$ and $20\%$ as train/test sets, and perform the three perturbation approaches on two decomposition methods under the appropriate parameters. For comparison, we use the vanilla decomposition methods without perturbation as baselines. We measure the performance of tensor completion using the Root Mean Square Error~(RMSE) metric, computed by $\operatorname{RMSE} = \sqrt{\sum_{\Omega} (\tilde{x}_{ijk} - \hat{x}_{ijk})^2 / |\Omega|}$, where $\Omega$ represents indexes of test set. Owing to the uncertainty of introducing noise, the reported RMSE is averaged across multiple runs.

\subsection{Simulation Study}

In this part, we set the size and rank of $\mathcal{X}$ to $20 \times 20 \times 20$ and $3$, respectively. We use different ways to generate target tensor for CP and Tucker decompositions. Motivated by~\cite{acar2011scalable}, we construct $\mathcal{X}$ for CP decomposition by $\mathcal{X} = \llbracket \widetilde{\mathbf{A}}, \widetilde{\mathbf{B}}, \widetilde{\mathbf{C}} \rrbracket + \mathcal{N}$ where $\widetilde{\mathbf{A}}\in \mathbb{R}^{20\times 3}, \widetilde{\mathbf{B}}\in \mathbb{R}^{20\times 3}$ and $\widetilde{\mathbf{C}}\in \mathbb{R}^{20\times 3}$ are from standard normal distribution, and $\mathcal{N}$ represents a mean zero Gaussian noise tensor satisfying that signal-to-noise~(SNR)~is one. In addition, all columns of the factor matrices are normalized to unit length. For Tucker decomposition, we generate the factor matrices by a similar manner and make their columns orthogonal to each other. We draw the entries of the core tensor $\widetilde{{\mathcal{G}}} \in \mathbb{R}^{3 \times 3 \times 3}$ from standard normal distribution~\cite{song2019tensor} and construct $\mathcal{X}$ via $\mathcal{X} = \llbracket \widetilde{\mathcal{G}} ; \widetilde{\mathbf{A}}, \widetilde{\mathbf{B}}, \widetilde{\mathbf{C}} \rrbracket + \mathcal{N}$ where $\mathcal{N}$ is same as the generation in CP decomposition. For the convenience of performance visualization, we transform $\widetilde{\mathcal{X}}$ by min-max scaling before introducing noise tensor.  With regard to parameters setting, we set regularization term $\lambda$ in~\eqref{eq: cp-optimization} to $0.01$, and learning rate to $0.005$ in CP decomposition. For Tucker decomposition, we take regularization terms $\lambda_o$ and $\lambda_g$ in~\eqref{eq: tucker-optimization} by $0.001$ and $0.0001$, respectively, and learning rate by $0.005$. For both decomposition methods, we the set maximum number of iterations to $100$. In the following experiments, we consider $\widetilde{\mathcal{X}}$ with missing ratio of $50\%$ as the benchmark case. 

\begin{figure}[t]
	\centering
	 \includegraphics[width=0.8\linewidth]{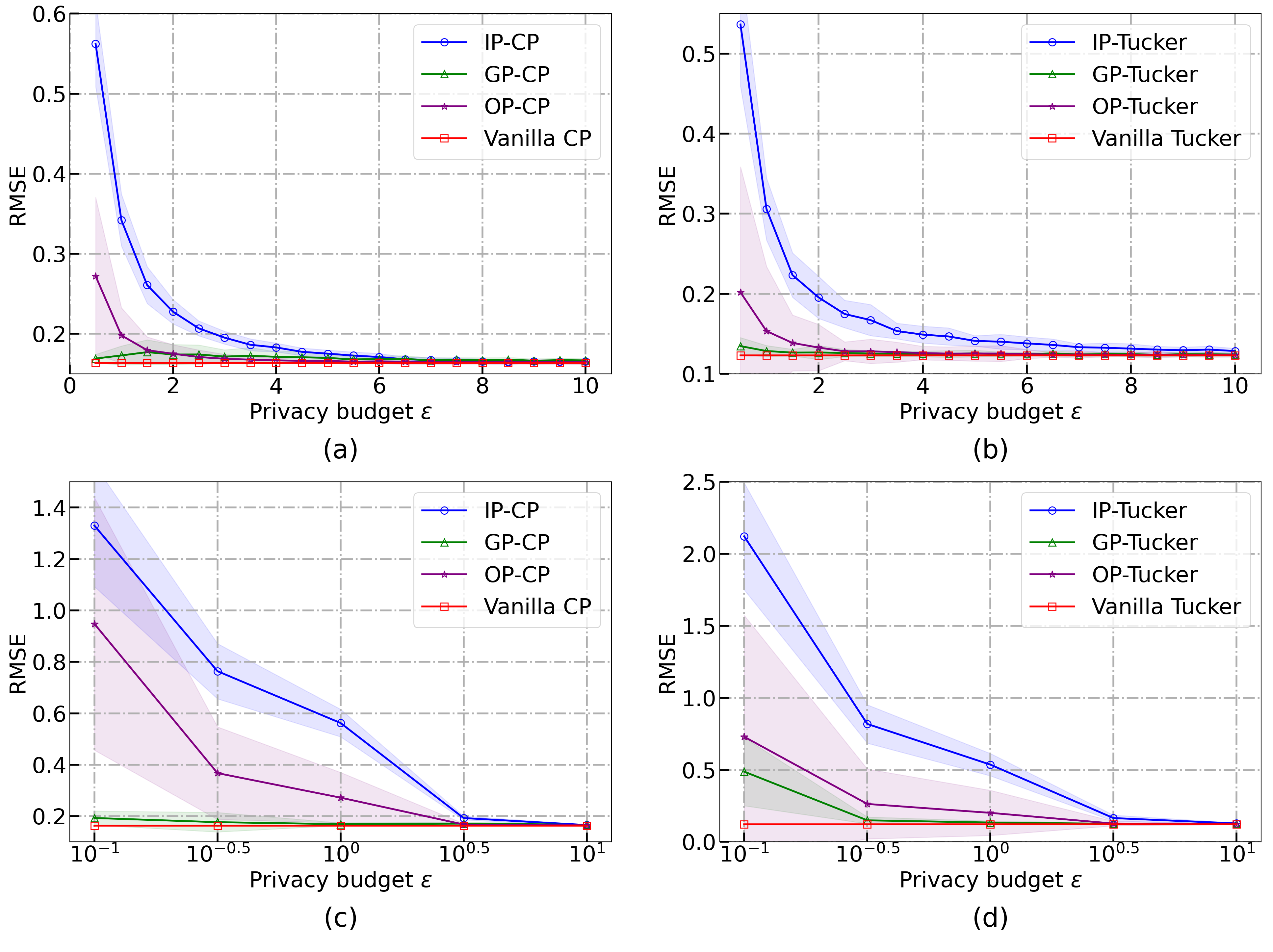}
    \caption{Performance comparison of CP and Tucker decompositions. The left and right columns present the performance of CP decomposition and Tucker decomposition, respectively. The colored area around each RMSE curve reflects the standard deviation of RMSE, averaged over $50$ realizations. We can view this area as a stability indicator for each perturbation approach.}
	\label{fig1}
\end{figure}

Figure~\ref{fig1} shows the performance comparisons among several perturbation approaches under the same decomposition method. As expected, decomposition methods with perturbation approaches cannot outperform the baselines, and their RMSE increase with the privacy parameter $\epsilon$ shrinking. This can be explained by that keeping a higher level of privacy means introducing larger noises, which leads to lower accuracy. Overall, there is no significant difference in the trade-off of privacy-accuracy between CP decomposition and Tucker decomposition. Specifically, the figure illustrates that the performance of gradient perturbation~(GP) is followed by output perturbation~(OP) and input perturbation~(IP) in terms of accuracy and stability. In Figure~\ref{fig1}~(a)~and~(b), we observe that the curves of GP are very close to that of the baselines, which is caused by the experimental setting where we set $\Delta^{(G)}_{\mathcal{X}} = 2m$, and $m$ here indicates the clipping constant. A small clipping constant means a small $\Delta^{(G)}_{\mathcal{X}}$, which can offset the impact of smaller $\epsilon$. We observe the trade-off of gradient perturbation in Figure~\ref{fig1}~(c)~and~(d), where privacy parameters are presented by exponential magnitude.
\begin{figure}[htbp]
	\centering
	 \includegraphics[width=0.95\linewidth]{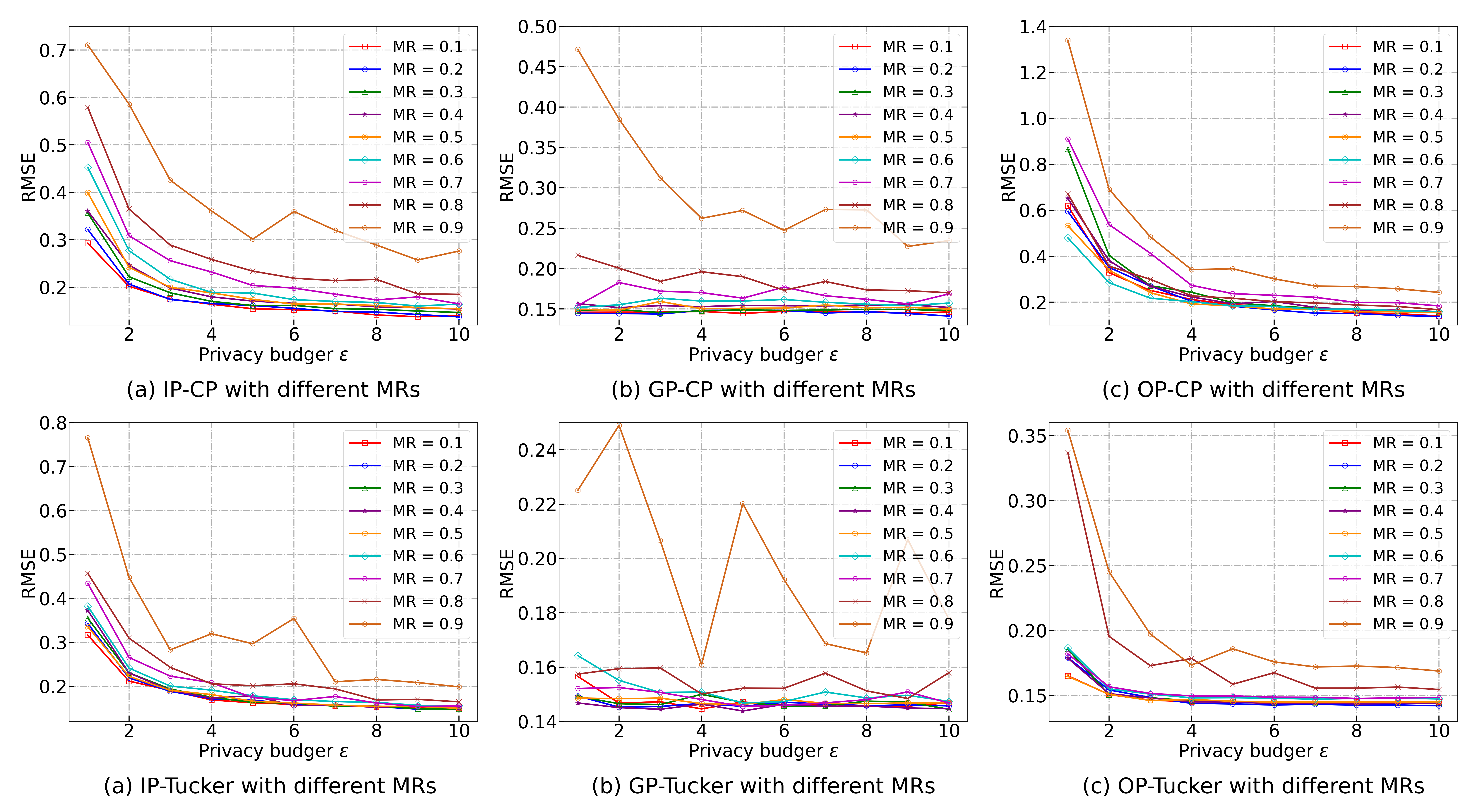}
    \caption{Performance comparison with different MR for the same perturbation approach. The first- and second-row are pertaining to the perturbation methods in CP decomposition and Tucker decomposition, respectively, where the displayed results are averaged over $10$ runs.}
	\label{fig2}
\end{figure}

In addition, from the perspective of stability, GP can maintain the lowest fluctuation as long as $\epsilon$ is set not too small. This is consistent with our analysis in Algorithm~\ref{Al2}. For two other perturbation methods, the fluctuation of IP keeps dropping mildly with $\epsilon$, whereas OP exhibits a sharp decrease. One reasonable explanation is that the iteration number of SGD has much difference in mitigating noise impact. Obviously, IP has the most iteration numbers since the noise is introduced prior to the iteration process. In contrast to IP, OP adds noise to the estimated output after completing the whole iteration process, making completion performance susceptible to  noise.

Figure~\ref{fig2} reveals that the comparison results of several perturbation approaches under a series of missing scenarios. For each perturbation approach, we evaluate the completion performance with missing ratio~(MR)~ranging from $0.1$ to $0.9$. Overall, the figure demonstrates an increased tendency of RMSE with the higher MR. In particular, a significant decrease in performance only occurs when the missing ratio is taken by $0.9$, which implies that our proposed privacy-preserving approaches can maintain the high accuracy until the sparsity of the dataset is over a certain high threshold.

\subsection{Empirical Study on ML-$100$K}

In this part, we analyze the performance of the proposed methods on MovieLens~$100$K~(ML-$100$K)~\cite{harper2015movielens}~datasets, which consists of $943$ users and $1682$ movies and has the density of $6.30\%$. We divide the timestamps into $212$ values by day and unfold the original rating matrix to tensor by expanding timestamps as the third dimension. We utilize the canonical partition~(ua.base/ua.test and ub.base/ub.test) to train and evaluate our proposed perturbation methods. To avoid the bias issue of data, we employ the bi-scaling procedure~\cite{mazumder2010spectral}, which standardizes a matrix to have rows and columns of means zero and variances one, to matrices separated from tensor by timestamp before applying any perturbation methods. In terms of parameters setting, we set $\lambda$ to $0.01$ in~\eqref{eq: cp-optimization} and the learning rate to $0.005$ in CP decomposition. Also, we take $\lambda_o = 0.01$, $\lambda_g = 0.001$ in~\eqref{eq: tucker-optimization} and the learning rate to $0.003$ in Tucker decomposition. For both methods, we set maximum number of iterations to $100$.

\begin{figure}[htbp]
	\centering
	 \includegraphics[width=.8\linewidth]{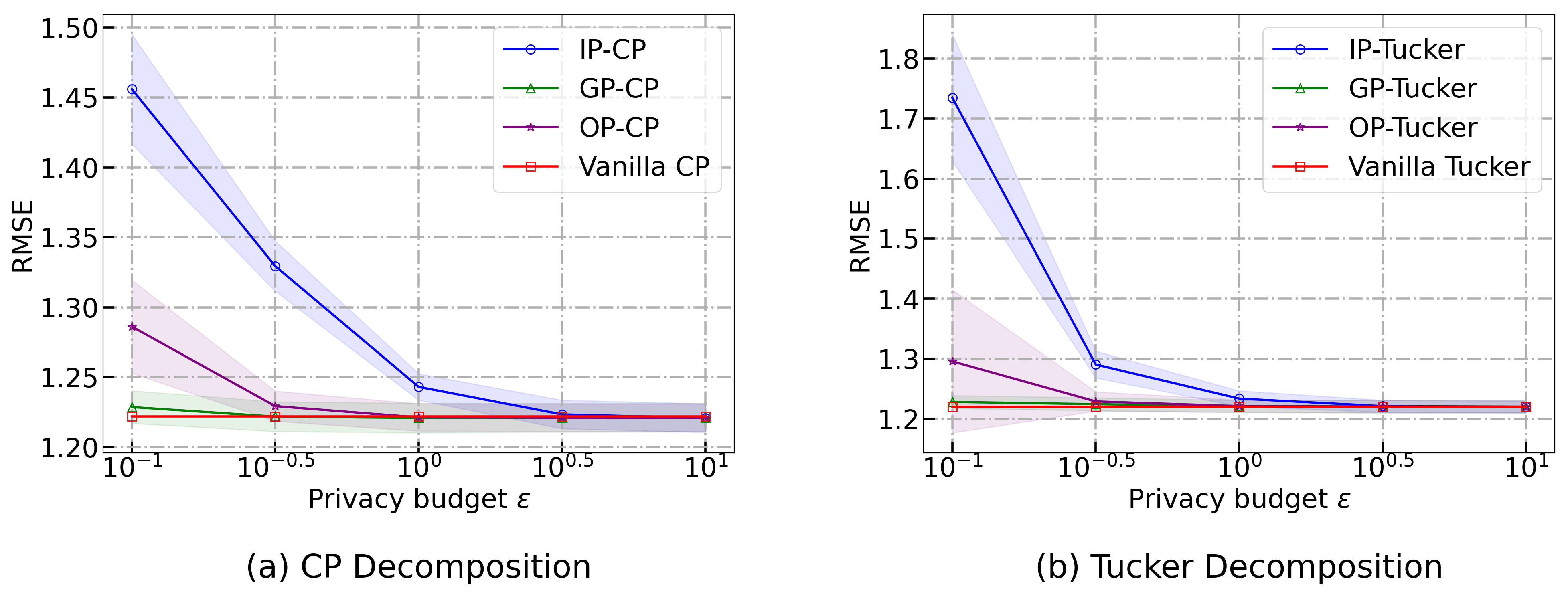}
    \caption{Comparison results on ML-$100$K under CP and Tucker decompositions. The reported results are the average of two splited datasets through 10 runs.}
	\label{fig3}
\end{figure}

Figure~\ref{fig3} shows the performance of perturbation approaches on ML-$100$K with the same decomposition method. Overall, we observe that three approaches have comparable performance to that on synthetic datasets, which validates the effectiveness of our proposal in practical scenarios.

\section{Conclusion and Future Work}\label{sec:con}
In this paper, we have established a unified privacy-preserving framework for CP and Tucker decompositions. This framework contains three perturbation approaches to tackle the privacy issue in tensor completion via differential privacy. For each approach, we have provided the algorithm procedures and theoretical analyses. Through experiments on synthetic and real-world datasets, we have verified the effectiveness of the proposed framework. Particularly worth mentioning is that the gradient perturbation approach can achieve a stable and remarkable accuracy with small privacy budgets, indicating great potential for practical applications.

There are many intriguing future directions to pursue. Firstly, we can adapt our proposal to improve variants of tensor completion, especially for methods based on CP or Tucker decompositions. Secondely, we can extend the framework to other scenarios where servers' responsible for data collection are untrusted. Thirdly, we can develop more sophisticated methods to incorporate the side information of target tensor in our proposed framework to obtain further performance enhancement.

\bibliographystyle{plain}
\bibliography{differential_privacy, tensor_completion, matrix_completion}

\begin{thebibliography}{10}

\bibitem{abadi2016deep}
Martin Abadi, Andy Chu, Ian Goodfellow, H~Brendan McMahan, Ilya Mironov, Kunal
  Talwar, and Li~Zhang.
\newblock Deep learning with differential privacy.
\newblock In {\em Proceedings of the 2016 ACM SIGSAC conference on computer and
  communications security}, pages 308--318, 2016.

\bibitem{acar2011scalable}
Evrim Acar, Daniel~M Dunlavy, Tamara~G Kolda, and Morten M{\o}rup.
\newblock Scalable tensor factorizations for incomplete data.
\newblock {\em Chemometrics and Intelligent Laboratory Systems}, 106(1):41--56,
  2011.

\bibitem{agrawal2000privacy}
Rakesh Agrawal and Ramakrishnan Srikant.
\newblock Privacy-preserving data mining.
\newblock In {\em Proceedings of the 2000 ACM SIGMOD international conference
  on Management of data}, pages 439--450, 2000.

\bibitem{aimeur2008lambic}
Esma A{\"\i}meur, Gilles Brassard, Jos{\'e}~M Fernandez, and Flavien Serge~Mani
  Onana.
\newblock A lambic: a privacy-preserving recommender system for electronic
  commerce.
\newblock {\em International Journal of Information Security}, 7(5):307--334,
  2008.

\bibitem{amit2007uncovering}
Yonatan Amit, Michael Fink, Nathan Srebro, and Shimon Ullman.
\newblock Uncovering shared structures in multiclass classification.
\newblock In {\em Proceedings of the 24th international conference on Machine
  learning}, pages 17--24, 2007.

\bibitem{bassily2014private}
Raef Bassily, Adam Smith, and Abhradeep Thakurta.
\newblock Private empirical risk minimization: Efficient algorithms and tight
  error bounds.
\newblock In {\em 2014 IEEE 55th Annual Symposium on Foundations of Computer
  Science}, pages 464--473. IEEE, 2014.

\bibitem{carroll1970analysis}
J~Douglas Carroll and Jih-Jie Chang.
\newblock Analysis of individual differences in multidimensional scaling via an
  n-way generalization of “eckart-young” decomposition.
\newblock {\em Psychometrika}, 35(3):283--319, 1970.

\bibitem{chaudhuri2008privacy}
Kamalika Chaudhuri and Claire Monteleoni.
\newblock Privacy-preserving logistic regression.
\newblock In {\em NIPS}, volume~8, pages 289--296. Citeseer, 2008.

\bibitem{chaudhuri2011differentially}
Kamalika Chaudhuri, Claire Monteleoni, and Anand~D Sarwate.
\newblock Differentially private empirical risk minimization.
\newblock {\em Journal of Machine Learning Research}, 12(3), 2011.

\bibitem{de2000multilinear}
Lieven De~Lathauwer, Bart De~Moor, and Joos Vandewalle.
\newblock A multilinear singular value decomposition.
\newblock {\em SIAM journal on Matrix Analysis and Applications},
  21(4):1253--1278, 2000.

\bibitem{dwork2006calibrating}
Cynthia Dwork, Frank McSherry, Kobbi Nissim, and Adam Smith.
\newblock Calibrating noise to sensitivity in private data analysis.
\newblock In {\em Theory of cryptography conference}, pages 265--284. Springer,
  2006.

\bibitem{dwork2014algorithmic}
Cynthia Dwork, Aaron Roth, et~al.
\newblock The algorithmic foundations of differential privacy.
\newblock {\em Foundations and Trends in Theoretical Computer Science},
  9(3-4):211--407, 2014.

\bibitem{evgeniou2007multi}
An~Evgeniou and Massimiliano Pontil.
\newblock Multi-task feature learning.
\newblock {\em Advances in neural information processing systems}, 19:41, 2007.

\bibitem{friedman2016differential}
Arik Friedman, Shlomo Berkovsky, and Mohamed~Ali Kaafar.
\newblock A differential privacy framework for matrix factorization recommender
  systems.
\newblock {\em User Modeling and User-Adapted Interaction}, 26(5):425--458,
  2016.

\bibitem{friedman2010data}
Arik Friedman and Assaf Schuster.
\newblock Data mining with differential privacy.
\newblock In {\em Proceedings of the 16th ACM SIGKDD international conference
  on Knowledge discovery and data mining}, pages 493--502, 2010.

\bibitem{goldberg1992using}
David Goldberg, David Nichols, Brian~M Oki, and Douglas Terry.
\newblock Using collaborative filtering to weave an information tapestry.
\newblock {\em Communications of the ACM}, 35(12):61--70, 1992.

\bibitem{harper2015movielens}
F~Maxwell Harper and Joseph~A Konstan.
\newblock The movielens datasets: History and context.
\newblock {\em Acm Transactions on interactive intelligent Systems (TiiS)},
  5(4):1--19, 2015.

\bibitem{harshman1970foundations}
Richard~A Harshman et~al.
\newblock Foundations of the parafac procedure: Models and conditions for an"
  explanatory" multimodal factor analysis.
\newblock 1970.

\bibitem{hitchcock1927expression}
Frank~L Hitchcock.
\newblock The expression of a tensor or a polyadic as a sum of products.
\newblock {\em Journal of Mathematics and Physics}, 6(1-4):164--189, 1927.

\bibitem{2015Differentially}
J.~Hua, X.~Chang, and Z.~Sheng.
\newblock Differentially private matrix factorization.
\newblock {\em AAAI Press}, 2015.

\bibitem{imtiaz2018distributed}
Hafiz Imtiaz and Anand~D Sarwate.
\newblock Distributed differentially private algorithms for matrix and tensor
  factorization.
\newblock {\em IEEE journal of selected topics in signal processing},
  12(6):1449--1464, 2018.

\bibitem{jain2012differentially}
Prateek Jain, Pravesh Kothari, and Abhradeep Thakurta.
\newblock Differentially private online learning.
\newblock In {\em Conference on Learning Theory}, pages 24--1. JMLR Workshop
  and Conference Proceedings, 2012.

\bibitem{jayaraman2018distributed}
Bargav Jayaraman, Lingxiao Wang, David Evans, and Quanquan Gu.
\newblock Distributed learning without distress: privacy-preserving empirical
  risk minimization.
\newblock In {\em Proceedings of the 32nd International Conference on Neural
  Information Processing Systems}, pages 6346--6357, 2018.

\bibitem{johnson2016mimic}
Alistair~EW Johnson, Tom~J Pollard, Lu~Shen, H~Lehman Li-Wei, Mengling Feng,
  Mohammad Ghassemi, Benjamin Moody, Peter Szolovits, Leo~Anthony Celi, and
  Roger~G Mark.
\newblock Mimic-iii, a freely accessible critical care database.
\newblock {\em Scientific data}, 3(1):1--9, 2016.

\bibitem{kolda2009tensor}
Tamara~G Kolda and Brett~W Bader.
\newblock Tensor decompositions and applications.
\newblock {\em SIAM review}, 51(3):455--500, 2009.

\bibitem{kroonenberg1980principal}
Pieter~M Kroonenberg and Jan De~Leeuw.
\newblock Principal component analysis of three-mode data by means of
  alternating least squares algorithms.
\newblock {\em Psychometrika}, 45(1):69--97, 1980.

\bibitem{liu2012tensor}
Ji~Liu, Przemyslaw Musialski, Peter Wonka, and Jieping Ye.
\newblock Tensor completion for estimating missing values in visual data.
\newblock {\em IEEE Transactions on Pattern Analysis and Machine Intelligence},
  35(1):208--220, 2013.

\bibitem{liu2015fast}
Ziqi Liu, Yu-Xiang Wang, and Alexander Smola.
\newblock Fast differentially private matrix factorization.
\newblock In {\em Proceedings of the 9th ACM Conference on Recommender
  Systems}, pages 171--178, 2015.

\bibitem{ma2019privacy}
Jing Ma, Qiuchen Zhang, Jian Lou, Joyce~C Ho, Li~Xiong, and Xiaoqian Jiang.
\newblock Privacy-preserving tensor factorization for collaborative health data
  analysis.
\newblock In {\em Proceedings of the 28th ACM International Conference on
  Information and Knowledge Management}, pages 1291--1300, 2019.

\bibitem{mazumder2010spectral}
Rahul Mazumder, Trevor Hastie, and Robert Tibshirani.
\newblock Spectral regularization algorithms for learning large incomplete
  matrices.
\newblock {\em The Journal of Machine Learning Research}, 11:2287--2322, 2010.

\bibitem{mcsherry2009differentially}
Frank McSherry and Ilya Mironov.
\newblock Differentially private recommender systems: Building privacy into the
  netflix prize contenders.
\newblock In {\em Proceedings of the 15th ACM SIGKDD international conference
  on Knowledge discovery and data mining}, pages 627--636, 2009.

\bibitem{morup2011applications}
Morten M{\o}rup.
\newblock Applications of tensor (multiway array) factorizations and
  decompositions in data mining.
\newblock {\em Wiley Interdisciplinary Reviews: Data Mining and Knowledge
  Discovery}, 1(1):24--40, 2011.

\bibitem{narayanan2008robust}
Arvind Narayanan and Vitaly Shmatikov.
\newblock Robust de-anonymization of large sparse datasets.
\newblock In {\em 2008 IEEE Symposium on Security and Privacy (sp 2008)}, pages
  111--125. IEEE, 2008.

\bibitem{schein2016bayesian}
Aaron Schein, Mingyuan Zhou, David Blei, and Hanna Wallach.
\newblock Bayesian poisson tucker decomposition for learning the structure of
  international relations.
\newblock In {\em International Conference on Machine Learning}, pages
  2810--2819. PMLR, 2016.

\bibitem{shen2014privacy}
Yilin Shen and Hongxia Jin.
\newblock Privacy-preserving personalized recommendation: An instance-based
  approach via differential privacy.
\newblock In {\em 2014 IEEE International Conference on Data Mining}, pages
  540--549. IEEE, 2014.

\bibitem{song2019tensor}
Qingquan Song, Hancheng Ge, James Caverlee, and Xia Hu.
\newblock Tensor completion algorithms in big data analytics.
\newblock {\em ACM Transactions on Knowledge Discovery from Data (TKDD)},
  13(1):1--48, 2019.

\bibitem{tan2014tensor}
Huachun Tan, Bin Cheng, Wuhong Wang, Yu-Jin Zhang, and Bin Ran.
\newblock Tensor completion via a multi-linear low-n-rank factorization model.
\newblock {\em Neurocomputing}, 133:161--169, 2014.

\bibitem{tomasi1992shape}
Carlo Tomasi and Takeo Kanade.
\newblock Shape and motion from image streams under orthography: a
  factorization method.
\newblock {\em International journal of computer vision}, 9(2):137--154, 1992.

\bibitem{tucker1966some}
Ledyard~R Tucker.
\newblock Some mathematical notes on three-mode factor analysis.
\newblock {\em Psychometrika}, 31(3):279--311, 1966.

\bibitem{vardi1996network}
Yehuda Vardi.
\newblock Network tomography: Estimating source-destination traffic intensities
  from link data.
\newblock {\em Journal of the American Statistical Association},
  91(433):365--377, 1996.

\bibitem{wang2019dp}
Bao Wang, Quanquan Gu, March Boedihardjo, Farzin Barekat, and Stanley~J Osher.
\newblock Dp-lssgd: A stochastic optimization method to lift the utility in
  privacy-preserving erm.
\newblock {\em arXiv preprint arXiv:1906.12056}, 2019.

\bibitem{wang2018differentially}
Di~Wang, Minwei Ye, and Jinhui Xu.
\newblock Differentially private empirical risk minimization revisited: faster
  and more general.
\newblock In {\em Proceedings of the 31st International Conference on Neural
  Information Processing Systems}, pages 2719--2728, 2017.

\bibitem{wang2016online}
Yining Wang and Animashree Anandkumar.
\newblock Online and differentially-private tensor decomposition.
\newblock In {\em Proceedings of the 30th International Conference on Neural
  Information Processing Systems}, pages 3539--3547, 2016.

\bibitem{williams2010probabilistic}
Oliver Williams and Frank McSherry.
\newblock Probabilistic inference and differential privacy.
\newblock In {\em Proceedings of the 23rd International Conference on Neural
  Information Processing Systems-Volume 2}, pages 2451--2459, 2010.

\bibitem{wright2009robust}
John Wright, Arvind Ganesh, Shankar Rao, and Yi~Ma.
\newblock Robust principal component analysis: Exact recovery of corrupted
  low-rank matrices via convex optimization.
\newblock {\em Coordinated Science Laboratory Report no. UILU-ENG-09-2210,
  DC-243}, 2009.

\bibitem{wu2017bolt}
Xi~Wu, Fengan Li, Arun Kumar, Kamalika Chaudhuri, Somesh Jha, and Jeffrey
  Naughton.
\newblock Bolt-on differential privacy for scalable stochastic gradient
  descent-based analytics.
\newblock In {\em Proceedings of the 2017 ACM International Conference on
  Management of Data}, pages 1307--1322, 2017.

\bibitem{xu2013parallel}
Yangyang Xu, Ruru Hao, Wotao Yin, and Zhixun Su.
\newblock Parallel matrix factorization for low-rank tensor completion.
\newblock {\em Inverse Problems \& Imaging}, 9(2):601, 2015.

\bibitem{yang2021optimized}
Jia Yang, Cai Fu, and Hongwei Lu.
\newblock Optimized and federated soft-impute for privacy-preserving tensor
  completion in cyber-physical-social systems.
\newblock {\em Information Sciences}, 564:103--123, 2021.

\bibitem{zhang2017efficient}
Jiaqi Zhang, Kai Zheng, Wenlong Mou, and Liwei Wang.
\newblock Efficient private erm for smooth objectives.
\newblock In {\em Proceedings of the 26th International Joint Conference on
  Artificial Intelligence}, pages 3922--3928, 2017.

\bibitem{zhou2017tensor}
Pan Zhou, Canyi Lu, Zhouchen Lin, and Chao Zhang.
\newblock Tensor factorization for low-rank tensor completion.
\newblock {\em IEEE Transactions on Image Processing}, 27(3):1152--1163, 2017.

\end{thebibliography}

\end{document}